%% file: main.tex
\documentclass{article}
\usepackage{iclr2027_preview}

\usepackage{amsmath,amssymb}
\newtheorem{proposition}{Proposition}
\usepackage{booktabs,multirow,tabularx}
\usepackage{placeins}
\usepackage{graphicx}
\usepackage{float}
\usepackage{capt-of}
\usepackage[protrusion=true,expansion=true]{microtype}
\usepackage{wrapfig}
\usepackage{needspace}
\usepackage{pifont}
\usepackage{xcolor}
\usepackage{colortbl}
\definecolor{citationblue}{HTML}{5186B5}
\input{tables/results_style}
\usepackage[colorlinks=true,citecolor=citationblue,linkcolor=black,urlcolor=black]{hyperref}
\usepackage{url}
\graphicspath{{figures/}}

\title{SMAT: Simple and Efficient\\Merge-Aware Training}
\author{%
Yanggan Gu\textsuperscript{1}\thanks{Equal contribution.}\quad
Yuanyi Wang\textsuperscript{1}\footnotemark[1]\quad
Zhen Li\textsuperscript{1}\quad
Shuo Cai\textsuperscript{1}\quad
Yuhang Liu\textsuperscript{1}\\[3pt]
\bfseries Junzhuo Li\textsuperscript{2}\quad
Zihao Wang\textsuperscript{3}\quad
Hongxia Yang\textsuperscript{1,4,5}\thanks{Corresponding author.}\\[7pt]
\normalfont\small\textsuperscript{1}The Hong Kong Polytechnic University (PolyU)\\[1pt]
\normalfont\small\textsuperscript{2}The Hong Kong University of Science and Technology (Guangzhou)\\[1pt]
\normalfont\small\textsuperscript{3}The Chinese University of Hong Kong\\[1pt]
\normalfont\small\textsuperscript{4}Research Institute of Generative Artificial Intelligence\\[1pt]
\normalfont\small\textsuperscript{5}InfiX.ai\quad
\urlstyle{same}\textbf{Code:}~\href{https://github.com/egangu/smat}{\textcolor{citationblue}{\underline{\nolinkurl{github.com/egangu/smat}}}}%
}

\hypersetup{
  pdftitle={SMAT: Simple and Efficient Merge-Aware Training},
  pdfauthor={Yanggan Gu, Yuanyi Wang, Zhen Li, Shuo Cai, Yuhang Liu, Junzhuo Li, Zihao Wang, Hongxia Yang},
  pdfsubject={Merge-aware training},
  pdfkeywords={model merging, merge-aware training, SMAT}
}

\begin{document}
\maketitle
\suppressfloats[t]
\pagestyle{fancy}

\begin{abstract}
Model merging integrates the capabilities of multiple experts without joint
retraining, but standard expert training optimizes task loss alone and does
not guarantee good performance after merging.
Merge-aware training (MAT) aims to improve merged performance, but existing
methods do not fully account for common merging operations and add training cost.
We observe that, from an expert's perspective, common merging methods can be
described by three operations: Scale reweights its own update, Mask removes
selected coordinates, and Perturb adds updates from other experts.
Based on this view, we introduce \textsc{SMAT} (Simple MAT), which jointly
optimizes expert loss and expected loss at simulated merged parameters
generated by sampling scaling coefficients, masks, and additive noise.
We further introduce periodic scheduling, kernel fusion, and parameter storage
switching to make SMAT efficient, with one forward and one backward pass per step.
Across four language and vision-language backbones, SMAT improves the mean
score across five merging methods by 1.07--2.16 points over the strongest
baseline for each backbone, with less than 2\% training-time overhead over
standard fine-tuning.
\end{abstract}

\input{figures/overview}
\input{sections/01_introduction}
\input{sections/02_related_work}
\input{sections/03_preliminaries}
\input{sections/04_method}
\input{sections/05_experiments}
\FloatBarrier
\input{sections/06_conclusion}
\input{sections/07_statements}

\section*{Acknowledgments}

This work was supported by the Hong Kong RGC (TRS: T41-517/25-N; GRF: 15228325), ITC RAISe+ (RAI/24/1/086A), and the Research Institute for Generative Artificial Intelligence at PolyU.

\bibliography{references,future_work}
\bibliographystyle{iclr2027_conference}

\clearpage
\appendix
\input{sections/appendix/01_experimental_settings}
\input{sections/appendix/02_analysis_settings}
\input{sections/appendix/03_theory}
\input{sections/appendix/04_training_algorithm}
\input{sections/appendix/09_implementation_measurement}
\clearpage
\input{sections/appendix/10_future_work}
\clearpage

\end{document}

%% file: tables/results_style.tex
\makeatletter
\newcommand{\smatgroupinset}{%
  \begingroup
  \ifx\CT@row@color\relax
    \hskip1.8pt\relax
  \else
    \let\CT@color\color
    \CT@row@color
    \vrule width 1.8pt\relax
  \fi
  \endgroup
}
\newcommand{\smatsetrulecolor}{%
  \let\smat@savedrulecolor\CT@arc@
  \arrayrulecolor{black!60}%
}
\newcommand{\smatrestorerulecolor}{%
  \global\let\CT@arc@\smat@savedrulecolor
}
\makeatother
\newcolumntype{G}{!{\smatgroupinset{\color{black!60}\vrule width 0.4pt}\smatgroupinset}}
\newlength{\smatresultscorewidth}
\newcommand{\smatresultsetup}{%
  \small
  \setlength{\tabcolsep}{2.2pt}%
  \setlength{\extrarowheight}{1pt}%
  \setlength{\arrayrulewidth}{0.4pt}%
  \setlength{\heavyrulewidth}{0.8pt}%
  \setlength{\lightrulewidth}{0.4pt}%
  \setlength{\cmidrulewidth}{0.4pt}%
  \smatsetrulecolor
  \setlength{\aboverulesep}{0pt}%
  \setlength{\belowrulesep}{0pt}%
  \renewcommand{\arraystretch}{1.12}%
  \settowidth{\smatresultscorewidth}{00.00}%
}
\newcommand{\smatcostunit}[1]{{\fontsize{6}{7}\selectfont\color{black!70}#1}}
\newcommand{\smatmergeheading}[1]{%
  \makebox[\smatresultscorewidth][c]{#1}%
  \phantom{\kern0.2pt{\scriptsize$\uparrow$0.0}}%
}
\newcommand{\smatavgheading}{\smatmergeheading{\textbf{Avg.}}}

%% file: figures/overview.tex
\begin{figure}[H]
  \centering
  \includegraphics[width=\linewidth]{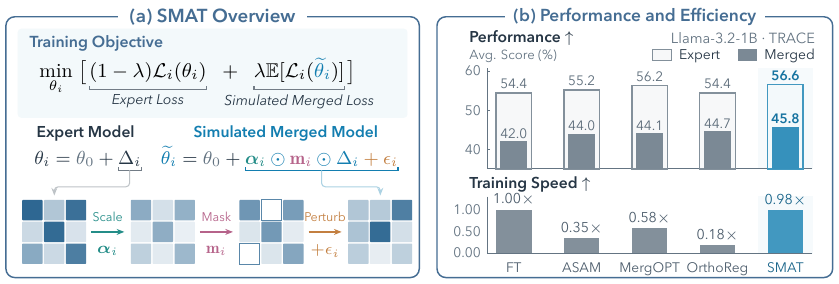}
  \caption{\textbf{Overview of SMAT. } (a) SMAT simulates merged parameter states during expert training through Scale, Mask, and Perturb. (b) This improves merged performance while retaining training speed close to standard fine-tuning.}
  \label{fig:overview}
\end{figure}

%% file: sections/01_introduction.tex
\section{Introduction}
\label{sec:introduction}

Model merging combines parameter updates from independently fine-tuned experts
that share a pretrained initialization, integrating their capabilities without
joint retraining~\citep{ilharco2023taskarithmetic}. Standard fine-tuning,
however, optimizes only each expert's task loss, which does not guarantee
good performance after merging. Overlapping updates and conflicting signs can
reduce task performance after merging~\citep{yadav2023ties}.

Merge-aware training (MAT) trains experts to perform well after merging.
SAFT~\citep{lee2025saft} trains for low loss near the expert's weights;
MergOPT~\citep{yang2026mergopt} adds noise to simulate merging;
and OrthoReg~\citep{liu2026orthoreg} encourages orthogonality within update matrices.
However, these methods do not fully consider common merging operations,
especially scaling and masking an expert's own update. They also increase
training cost through extra passes, perturbation processing, or matrix regularization.

We observe that Task Arithmetic~\citep{ilharco2023taskarithmetic},
TIES-Merging~\citep{yadav2023ties}, DARE~\citep{yu2024dare}, and
DELLA~\citep{deep2024della} can be described by three operations from an
expert's perspective:
\emph{Scale} reweights its own update, \emph{Mask} removes selected
coordinates, and \emph{Perturb} adds updates from other experts.
Let $\theta_0$ be the pretrained parameters, $\theta_i$ the parameters of
expert $i$, and $\Delta_i=\theta_i-\theta_0$ its task vector.
From expert $i$'s view, the merged parameters are
\begin{equation}
    \theta_{\mathrm{merge}}
    =\theta_0+\boldsymbol\alpha_i\odot\mathbf m_i\odot\Delta_i+\epsilon_i,
    \qquad
    \epsilon_i=\sum_{j\ne i}\boldsymbol\alpha_j\odot\mathbf m_j\odot\Delta_j.
    \label{eq:merged-state}
\end{equation}
Here $\boldsymbol\alpha_i$ gives the merging weights,
$\mathbf m_i$ is the mask that selects retained updates,
and $\epsilon_i$ is the combined update from other experts;
$\odot$ denotes coordinatewise multiplication.

Motivated by these three merging operations, we introduce \textsc{SMAT}
(Simple MAT), illustrated in Figure~\ref{fig:overview}. Its objective jointly
optimizes the expert loss and the expected loss at simulated merged parameters to
improve both expert and merged performance. Following Eq.~\eqref{eq:merged-state},
SMAT samples scaling coefficients and masks for its own update and
noise for other experts' updates.

We further introduce three improvements for efficient SMAT training.
\emph{Periodic scheduling} applies the simulated merged loss every few steps
and the expert loss otherwise, keeping one forward and one backward pass per step.
\emph{Kernel fusion} computes Scale, Mask, and Perturb together to
reduce repeated parameter reads and writes. \emph{Parameter storage switching}
keeps the original and simulated weights in separate buffers and switches
between them to avoid repeated weight copying.

Our experiments cover language and vision-language models ranging from
151M to 8B parameters, evaluated with five merging methods.
Across all four backbones, SMAT achieves better merged performance and
shorter training time than the compared MAT methods, with less than 2\%
overhead over standard fine-tuning. For Llama-3.2-1B-Instruct, SMAT improves merged
performance by 1.12 points over the strongest baseline, OrthoReg, while
using 82\% less training time.
Further analysis shows that SMAT broadens low-loss regions along merge-relevant directions, providing a loss-smoothing perspective on its improved robustness to merging.

We summarize our three contributions:\par\nopagebreak[4]

\noindent\ding{182}~We formulate merge-aware training through Scale, Mask, and
Perturb, simulating these common merging operations during expert training.

\smallskip
\noindent\ding{183}~We implement this objective with one forward and one
backward pass per step, using periodic scheduling and efficient parameter operations.

\smallskip
\noindent\ding{184}~We demonstrate improved merging across language and vision-language
models with less than 2\% training-time overhead, examine merging behavior,
and extend SMAT to Muon.

%% file: sections/02_related_work.tex
\section{Related Work}
\label{sec:related_work}

\paragraph{Model Merging.}
\label{sec:related-merging}
Stochastic Weight Averaging~\citep{izmailov2018swa},
Model Stock~\citep{jang2024modelstock}, and
Model Ratatouille~\citep{rame2023ratatouille} show that weight averaging can
improve generalization.
Model Soups~\citep{wortsman2022soups} averages model weights, while Fisher
Merging~\citep{matena2022fisher} gives more important parameters more weight.
Task Arithmetic~\citep{ilharco2023taskarithmetic} combines updates from a shared
initialization, and AdaMerging~\citep{yang2024adamerging} learns coefficients
for tasks or layers.
TIES-Merging~\citep{yadav2023ties}, DARE~\citep{yu2024dare},
DELLA~\citep{deep2024della}, and Model Breadcrumbs~\citep{davari2024breadcrumbs}
use magnitude, sign agreement, or random deletion to select updates.
Other methods learn shared masks or retain task-specific masks and
rescalers~\citep{tang2023concrete,wang2024tall,huang2024emr}.
RegMean~\citep{jin2023regmean} uses linear regression to preserve layer
outputs, while TSV-Merge~\citep{gargiulo2025tsv} uses subspaces defined by
the singular vectors of task updates.
These methods change task vectors during merging. SMAT simulates their
scaling, masking, and additive changes during expert training.

\paragraph{Merge-Aware Training.}
\label{sec:related-mat}
Merge-aware training (MAT) changes expert training to improve performance
after merging. Tangent-space Task Arithmetic~\citep{ortiz2023tangent} uses
linearized fine-tuning for this purpose.
SAFT~\citep{lee2025saft} uses sharpness-aware fine-tuning to
find locally flatter solutions, building on SAM~\citep{foret2021sam} and
ASAM~\citep{kwon2021asam}. MergOPT~\citep{yang2026mergopt} trains at
sampled merge offsets, and OrthoReg~\citep{liu2026orthoreg} encourages
orthogonality within update matrices.
However, these methods do not fully consider common merging operations,
particularly the scaling and masking of an expert's own task vector.
They also incur extra computation through additional gradient evaluations,
perturbation processing, or matrix regularization.
\textsc{SMAT} trains experts under these changes by simulating Scale, Mask, and Perturb. Periodic scheduling and fused
operations reduce overhead.

%% file: sections/03_preliminaries.tex
\section{Preliminaries}
\label{sec:preliminaries}
\subsection{Problem Setup}
\label{sec:problem-setup}

We consider $N$ tasks with a shared pretrained initialization
$\theta_0\in\mathbb{R}^d$, where $d$ is the number of trainable shared
parameters that we merge and perturb.
Expert $i$ trains on its data distribution $\mathcal{D}_i$, without access
to other experts' data or checkpoints.

Let $x$ and $y$ be an input and its target, $f_i(x;\theta)$ expert $i$'s
predictor with parameters $\theta$, and $\ell_i$ its per-example loss. Its
expected task loss is
\begin{equation}
    \mathcal{L}_i(\theta)
    = \mathbb{E}_{(x,y)\sim\mathcal{D}_i}
      \big[\ell_i(f_i(x;\theta),y)\big],
    \label{eq:task-loss}
\end{equation}
where $\mathbb{E}$ averages over the indicated distribution.
Standard fine-tuning minimizes $\mathcal{L}_i$ to obtain $\theta_i$ for each expert.
The shared parameters are then merged; any task-specific heads stay separate.

\subsection{Task Vectors and Model Merging}
\label{sec:task-vector-merging}

Expert $i$'s task vector is $\Delta_i=\theta_i-\theta_0$, where $\theta_i$
denotes its fine-tuned parameters.
Equation~\eqref{eq:merged-state} weights and masks task vectors before
adding them to the shared initialization. The weights
$\boldsymbol\alpha_j$ can vary by parameter, and both weights and masks can
depend on all experts. This form covers Model Soups~\citep{wortsman2022soups},
Fisher Merging~\citep{matena2022fisher}, and Task Arithmetic~\citep{ilharco2023taskarithmetic},
as well as masking in TIES~\citep{yadav2023ties}, DARE~\citep{yu2024dare},
and DELLA~\citep{deep2024della}.

We use $\gamma$ for a scalar merging coefficient and $\mathbf1$ for the
all-ones vector. Task Arithmetic sets $\boldsymbol\alpha_j=\gamma\mathbf1$
and $\mathbf m_j=\mathbf1$; weight averaging instead uses
$\boldsymbol\alpha_j=\mathbf1/N$.

Relative to expert $i$, merging changes the parameters by
\begin{equation}
    \theta_{\mathrm{merge}}-\theta_i
    = (\boldsymbol\alpha_i\odot\mathbf m_i-\mathbf1)\odot\Delta_i+\epsilon_i.
    \label{eq:expert-displacement}
\end{equation}
Scale and Mask change the expert's own update, while Perturb represents the
other experts' contribution $\epsilon_i$. SMAT simulates these three changes
during independent training.

\subsection{The MAT Objective}
\label{sec:mat-objective}

Training should preserve each expert's performance while improving the merged
model. We express this goal with the MAT objective
\begin{equation}
    \mathcal{J}_i^{\mathrm{MAT}}
    = (1-\lambda)\mathcal{L}_i(\theta_i)
       +\lambda\mathcal{L}_i(\theta_{\mathrm{merge}}),
       \qquad 0\leq\lambda\leq1.
    \label{eq:ideal-mat}
\end{equation}
Here $\mathcal{J}_i^{\mathrm{MAT}}$ is expert $i$'s training objective, and
$\lambda$ weights the merged loss. The endpoints $\lambda=0$ and $\lambda=1$
train only for expert performance and merged performance, respectively.

%% file: sections/04_method.tex
\section{Method}
\label{sec:method}

\input{figures/operators}

\subsection{Simulating Merged Parameter States}
\label{sec:merge-simulation}

Following Eq.~\eqref{eq:merged-state}, SMAT simulates merged parameters
during expert training as
\begin{equation}
    \widetilde\theta_i
    =\theta_0+\boldsymbol\alpha_i\odot\mathbf m_i\odot\Delta_i+\epsilon_i.
    \label{eq:smat-state}
\end{equation}
We use $\boldsymbol\alpha_i=\alpha_i\mathbf1$, with one scalar coefficient
shared by all coordinates. Here $\mathbf m_i$ is a sampled mask, and
$\epsilon_i$ is sampled noise representing the combined updates from other
experts. We sample $\alpha_i$, $\mathbf m_i$, and $\epsilon_i$ independently.
Figure~\ref{fig:operators} illustrates these operations.

\paragraph{Scale.}
Merging may assign different weights to an expert's task vector, as in
WiSE-FT~\citep{wortsman2022wiseft} and Task Arithmetic~\citep{ilharco2023taskarithmetic}.
Scale samples a shared coefficient $\alpha_i\sim\mathcal{U}[\alpha_{\min},1]$,
where $\mathcal{U}$ is the uniform distribution and
$0\leq\alpha_{\min}\leq1$ is the smallest allowed coefficient.

\paragraph{Mask.}
Mask prepares the expert for coordinate removal using the random dropping
and rescaling of DARE~\citep{yu2024dare}.
This is related to stochastic masking in Dropout~\citep{srivastava2014dropout}
and DropConnect~\citep{wan2013dropconnect}, and to mixing weights with their
pretrained values in Mixout~\citep{lee2020mixout}.
Let $\mathcal S$ be the set of coordinates to mask: by default, linear
weights in Transformer attention and MLP blocks. Let $p$ be their drop
probability. For coordinate $k$, set $p_k=p$ inside $\mathcal S$ and $p_k=0$
outside it, and draw
\begin{equation}
    m_{i,k}=\begin{cases}
        0, & \text{with probability }p_k,\\
        \dfrac{1}{1-p_k}, & \text{with probability }1-p_k,
    \end{cases}
    \qquad 0\leq p<1.
    \label{eq:smat-mask}
\end{equation}

Here $m_{i,k}$ is the $k$-th mask entry and $\mathbf p$ collects the drop
probabilities. Masks are independent across coordinates; biases, normalization
parameters, and embeddings are excluded. Since $\mathbb{E}[m_{i,k}]=1$,
rescaling preserves the expected scaled update. Mask acts only on the task
vector, leaving pretrained weights and additive noise unchanged.

\noindent
\begin{minipage}[t]{.53\linewidth}
\vspace{0pt}
\paragraph{Perturb.}
Perturb simulates other experts' additive updates $\epsilon_i$ in
Eq.~\eqref{eq:merged-state}. Their checkpoints and task vectors are
unavailable during independent training, so we approximate the unknown
update with random parameter noise.

\smallskip
With Perturb alone on Llama-3.2-1B-Instruct with AdamW, we compare uniform, Gaussian, and
Laplace noise at matched root mean square (RMS)
(Figure~\ref{fig:noise-distributions}). Their trends are broadly similar,
and uniform gives the highest merged performance when RMS $=2 \times 10^{-3}$.
Laplace, motivated by task-vector statistics in
MergOPT~\citep{yang2026mergopt}, has no consistent advantage.
This suggests that exposure to additive parameter variation may matter
more than precise distribution matching in this setting.
\end{minipage}\hfill
\begin{minipage}[t]{.45\linewidth}
\vspace{0pt}
\input{figures/noise_distributions}
\end{minipage}
\par

This observation motivates a robustness view through loss
smoothing~\citep{wen2018smoothout}. With Scale and Mask fixed, the expected
perturbed loss averages task loss over nearby parameters, encouraging
tolerance to additive merge offsets. Appendix~\ref{app:perturb-smoothing}
expresses this as convolution; under its smoothness assumptions, the
leading small-noise correction depends on local curvature. Perturb may
therefore improve merge robustness without matching the full distribution
of other experts' updates. We independently draw
$\epsilon_{i,k}\sim\mathcal U[-\sqrt3\sigma,\sqrt3\sigma]$
at each trainable coordinate, with fixed RMS $\sigma$, zero mean, and
variance $\sigma^2$.

\subsection{A Practical Training Objective}
\label{sec:practical-objective}

We approximate the merged loss in Eq.~\eqref{eq:ideal-mat} by an expectation
over simulated parameters:
\begin{equation}
    \mathcal{J}_i^{\mathrm{SMAT}}
    = (1-\lambda)\mathcal{L}_i(\theta_i)
       +\lambda\mathbb{E}_{\alpha_i,\mathbf{m}_i,\epsilon_i}
          \big[\mathcal{L}_i(\widetilde{\theta}_i)\big].
    \label{eq:smat-objective}
\end{equation}
Here $\mathcal J_i^{\mathrm{SMAT}}$ is the training objective; $\lambda$
weights the simulated loss, and the expectation averages the three random draws.
At each simulated-loss update, one sampled state provides a stochastic estimate of this expectation, so training does not enumerate merging configurations or expert combinations.

How closely this objective represents actual merging depends on the sampled states.
Appendix~\ref{app:perturb-smoothing} relates the loss difference to differences
in the state means and second moments.

\paragraph{Gradient Analysis.}
Let $B$ be a sampled minibatch and $\mathcal L_{i,B}$ its average task loss.
With the three draws held fixed, the chain rule gives the gradient ($\nabla$)
with respect to the expert parameters $\theta_i$:
\begin{equation}
    \nabla_{\theta_i}\mathcal{L}_{i,B}
      (\widetilde{\theta}_i)
    = \alpha_i\mathbf m_i\odot
      \nabla_{\widetilde{\theta}_i}
        \mathcal{L}_{i,B}(\widetilde{\theta}_i).
    \label{eq:smat-gradient}
\end{equation}
The two factors have a simple meaning: $\alpha_i\mathbf m_i$ is expert
$i$'s participation coefficient at each coordinate, and the second factor is the simulated
merged model's gradient on task $i$. SMAT simulates merged weights with
Scale, Mask, and Perturb, then uses the participation coefficient to weight
the gradient passed back to the expert. If $\alpha_i=0$ or $m_{i,k}=0$,
the expert's $k$-th task-vector coordinate does not contribute to the simulated
merged model and receives zero gradient from the simulated loss.

\input{figures/efficiency}
\subsection{Scheduling and Efficient Implementation}
\label{sec:efficient-training}

Directly optimizing Eq.~\eqref{eq:smat-objective} evaluates both losses at
every step. These mixed-loss updates can instead be approximated by separate
updates on each loss, with step counts proportional to their weights.
Let $\mathcal L_A$ denote the expert loss and $\mathcal L_B$ the expected
simulated loss.

\begin{proposition}[Replacing mixed-loss updates]
\label{prop:periodic-approximation}
Let $J=(1-k/n)\mathcal L_A+(k/n)\mathcal L_B$, with integers $n\geq1$
and $0\leq k\leq n$. Starting from the same parameters and using a constant
step size $\eta>0$, replace $n$ full-gradient descent steps on $J$ with
$n-k$ steps on $\mathcal L_A$ and $k$ steps on $\mathcal L_B$ in any order.
If both loss gradients are $L_g$-Lipschitz and bounded in norm by $G$
in a neighborhood containing both trajectories, the resulting parameters satisfy
\begin{equation}
 \|\theta_n^{\mathrm{mixed}}-\theta_n^{\mathrm{separate}}\|_2
 \leq \eta^2 L_g G\,n(n-1).
 \label{eq:periodic-approximation}
\end{equation}
\end{proposition}

Thus, mixed-loss training can be replaced to first order by a short sequence
of single-loss steps. Appendix~\ref{app:training-algorithm} gives the proof
and SMAT's cycle-averaged gradient bound. This motivates periodic scheduling;
kernel fusion and parameter storage switching further reduce its overhead
(Figure~\ref{fig:efficient-training}).

\paragraph{Periodic scheduling.}
Each cycle has $t-1$ expert-loss updates and one simulated-loss update,
matching $\lambda=1/t$ in Eq.~\eqref{eq:smat-objective}. We use $t=4$,
corresponding to three expert-loss steps and one simulated-loss step.
Each step computes only one loss, using one forward and one backward pass.

\paragraph{Kernel fusion.}
Separate Scale, Mask, and Perturb operations repeatedly access
parameter-sized tensors. Two Triton kernels fuse weight simulation and
gradient rescaling (Eq.~\eqref{eq:smat-gradient}). They sample noise and masks
and group compatible tensors, reducing memory traffic and kernel launches.

\paragraph{Parameter storage switching.}
SMAT uses simulated weights in a reusable buffer for the forward and backward
pass, then restores the original storage for the optimizer update.
Frozen base weights are prefetched from pinned CPU memory into this buffer
during ordinary steps. This needs one additional GPU buffer the size of the
trainable weights.

%% file: figures/operators.tex
\begin{figure}[!t]
  \centering
  \includegraphics[width=\linewidth]{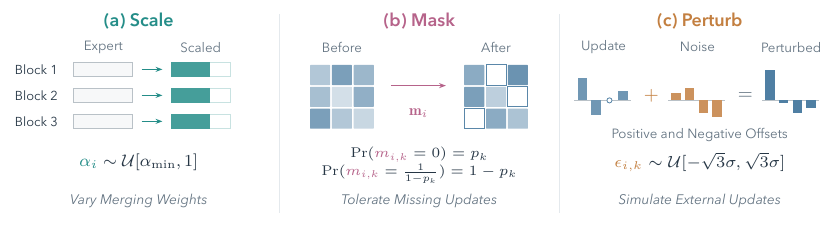}
  \caption{\textbf{Simulating merged updates.} Scale reweights the task
  vector, Mask drops and rescales its coordinates, and Perturb adds uniform noise to
  simulate other experts' updates.}
  \label{fig:operators}
\end{figure}

%% file: figures/noise_distributions.tex
  \centering
  \includegraphics[width=\linewidth]{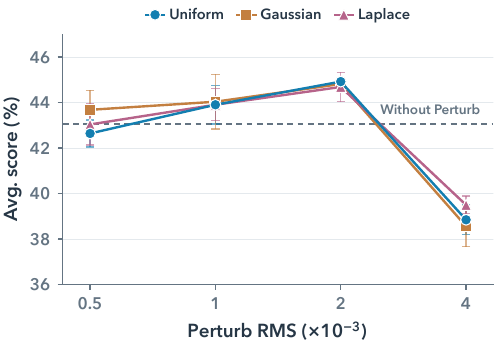}
  {\small\captionof{figure}{Perturb alone on Llama-1B (AdamW, $t=4$).
  Five-merger means $\pm$ SD.}
  \label{fig:noise-distributions}}

%% file: figures/efficiency.tex
\begin{figure}[!t]
  \centering
  \includegraphics[width=\linewidth]{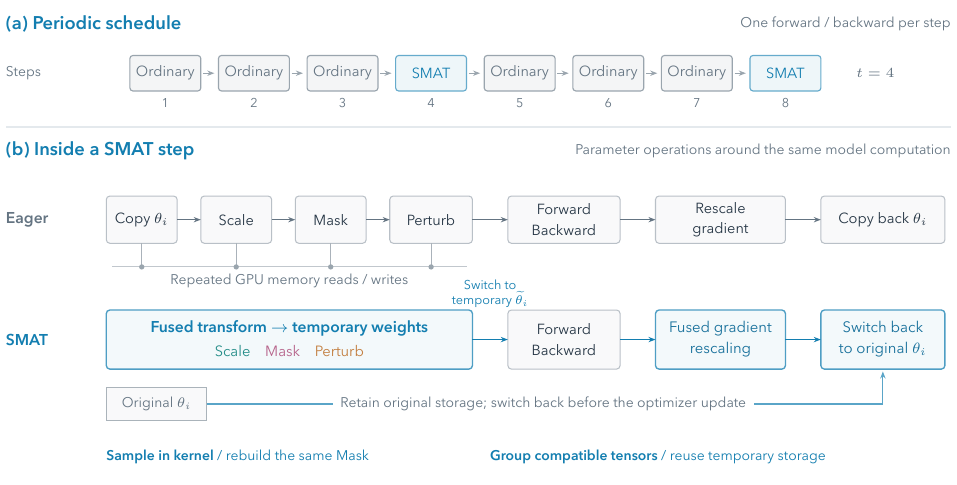}
  \caption{\textbf{Efficient SMAT training.} Each period of $t=4$ steps includes one
  SMAT update. Fused kernels and storage switching reduce parameter overhead.}
  \label{fig:efficient-training}
\end{figure}

%% file: sections/05_experiments.tex
\section{Experiments}
\label{sec:experiments}

\subsection{Experimental Setup}
\label{sec:experimental-setup}
\paragraph{Models and Tasks.}
We evaluate Llama-3.2-1B-Instruct and Llama-3.1-8B-Instruct from the
Llama 3 family~\citep{grattafiori2024llama3} on seven
TRACE~\citep{wang2023trace} tasks: C-STANCE, FOMC, MeetingBank, ScienceQA,
NumGLUE-cm, NumGLUE-ds, and 20Minuten. For vision, we evaluate
CLIP~\citep{radford2021clip} with ViT~\citep{dosovitskiy2021vit}
encoders (ViT-B/32 and ViT-L/14) on Cars, DTD, EuroSAT,
GTSRB, MNIST, RESISC45, SUN397, and SVHN, following
Task Arithmetic~\citep{ilharco2023taskarithmetic}.

\paragraph{Compared Methods.}
We compare SMAT with standard fine-tuning (FT),
ASAM~\citep{kwon2021asam} implemented as in SAFT~\citep{lee2025saft},
MergOPT~\citep{yang2026mergopt}, and OrthoReg~\citep{liu2026orthoreg}.
Each expert set is merged using weight averaging (WA),
Task Arithmetic~\citep{ilharco2023taskarithmetic} (TA),
TIES~\citep{yadav2023ties}, DARE~\citep{yu2024dare} followed by TA,
and DELLA~\citep{deep2024della}.

\paragraph{Training.}
Within each backbone, all methods share data, initialization, learning rate,
batch size, and training budget.
For language models, we follow the base training settings of
MergOPT~\citep{yang2026mergopt}: AdamW~\citep{loshchilov2019adamw} with
learning rate $2\times10^{-5}$, batch size 8, and task-specific epoch counts.
For vision models, we follow the public CLIP fine-tuning configuration of
FusionBench~\citep{tang2025fusionbench}: Adam~\citep{kingma2015adam} with
learning rate $10^{-5}$, batch size 128, and 4,000 updates per task.
We update all language-model parameters and only the CLIP vision encoder.
SMAT uses fused operations and simulates every $t=4$ steps, with Scale sampled from
$\mathcal{U}[\alpha_{\min},1]$, and Mask probability $p=0.5$ on
attention and MLP linear weights. Both language backbones use
$(\alpha_{\min},\sigma)=(0.2,2\times10^{-3})$; both vision backbones use
$(0.1,10^{-3})$, where $\alpha_{\min}$ is the Scale lower bound and $\sigma$
is the Perturb RMS. Each model with a MAT method is trained and evaluated three times using different random seeds, and the average performance is reported as the experimental result.

\paragraph{Dev Search.}
All methods share a development (dev) split held out from training data.
The dev and final evaluation splits have no overlapping examples.
We use dev data to select merging configurations.
For each backbone and training method, TA, TIES, DARE, and DELLA select
$\gamma\in\{0.1,0.2,\ldots,1.0\}$ by mean dev score, breaking ties in
favor of smaller values. WA requires no search. On TRACE, the top two
candidates are ranked by their mean over three decoding repeats.
Sparsity is calibrated on FT dev results and shared within each model family
(Appendix~\ref{app:current-main-settings}).

\paragraph{Evaluation.}
Scores are task means on a 0--100 scale: ROUGE-L~\citep{lin2004rouge} for
MeetingBank, SARI~\citep{xu2016sari} for 20Minuten, and accuracy otherwise.
We report final performance on the released eval split for language, with
three decoding repeats per checkpoint, and on the test split for vision.
Expert scores evaluate each expert
on its own task; Avg. is the mean merged score over five mergers.
Appendix~\ref{app:current-main-settings} gives the data splits and
configuration selection rules.

\input{tables/llm_results}
\input{tables/vit_results}
\subsection{Main Results}
\label{sec:main-results}

\paragraph{Language Models.}
\label{sec:language-results}
Table~\ref{tab:llm-results} shows that SMAT achieves Avg. scores across five mergers
of 45.77 on Llama-1B and 58.49 on Llama-8B, exceeding the strongest baseline,
OrthoReg, by 1.12 and 1.07 points, respectively.
It leads three of five merging columns on 1B and all five on 8B.
Expert performance improves over FT on 1B but decreases on 8B, showing that
better merged performance need not coincide with a better independent expert.

\paragraph{Vision-Language Models.}
\label{sec:vision-results}
In Table~\ref{tab:vit-results}, SMAT leads all five merging columns on both
CLIP backbones. Its Avg. across five mergers is 74.46 on ViT-B/32,
1.89 points above MergOPT (72.57), and 87.78 on ViT-L/14,
2.16 points above OrthoReg (85.62), the strongest baseline in each setting.
We use the same SMAT hyperparameters for both model sizes.

\paragraph{Training Efficiency.}
With frozen-base offload, SMAT adds about 0.2--1.9\% training-time overhead
across the four backbones, while increasing peak GPU memory by 1.7--24.3\%,
as shown in Tables~\ref{tab:llm-results} and~\ref{tab:vit-results}.
We measure expert training loops; the speed in Figure~\ref{fig:overview}
is FT time divided by each method's time. Appendix~\ref{app:training-time}
gives absolute costs and the software difference affecting ASAM timing.
\subsection{Analysis}
\label{sec:experimental-analysis}
We study SMAT's generalization and underlying mechanisms on Llama-1B using TRACE.

\noindent
\begin{minipage}[t]{.48\linewidth}
\vspace{0pt}
\paragraph{Generalization to Muon.}
\label{sec:muon-results}
To examine whether SMAT's gains extend beyond AdamW, we evaluate it with Muon on Llama-1B.
SMAT achieves an average score of 44.15 across five merging methods, outperforming OrthoReg by 0.74 points and standard fine-tuning by 5.90 points, while leading on four of the five mergers.
Its independent-expert performance remains close to FT (55.07 vs.\ 55.31), and its measured training time is $1.00\times$ that of FT.
These results show that SMAT's merging gains and low training overhead extend to Muon in this setting.
Full settings and results are provided in ~\ref{app:muon-refinement}.
\end{minipage}\hfill
\begin{minipage}[t]{.50\linewidth}
\vspace{0pt}
\input{tables/muon_results}
\end{minipage}
\par\medskip

\input{figures/expert_analyses}


\paragraph{Effect of the Number of Experts.}
\label{sec:expert-count}
We merge $K=2,\ldots,7$ experts with TA for FT, MergOPT, OrthoReg,
and SMAT, keeping each method's Table~\ref{tab:llm-results} TA coefficient.
Seven cyclic subsets balance task inclusion for $K<7$; $K=7$ uses all tasks.
We normalize scores by the corresponding FT expert scores, then average over
tasks and subsets (Appendix~\ref{app:expert-count-details}).
SMAT performs similarly to MergOPT with fewer experts and leads all three
baselines at $K=5,6,7$ in Figure~\ref{fig:expert-analyses}(a).
Its normalized score decreases from 93.20\% with two experts to 86.08\%
with all seven.
SMAT's advantage becomes clearer with more experts, 
suggesting greater robustness to the combined updates in larger expert sets.

\paragraph{Compatibility with Standard Experts.}
\label{sec:standard-expert-compatibility}
Figure~\ref{fig:expert-analyses}(b) replaces $q$ of seven FT experts with
MergOPT, OrthoReg, or SMAT experts for the same tasks. We use identical balanced
replacement sets and Table~\ref{tab:llm-results} training configurations.
All curves start from all-FT and fix the TA coefficient at $\gamma=0.3$.
The normalized score rises from 77.55\% for all-FT to 80.97\%, 83.51\%,
and 86.08\% for all-MergOPT, all-OrthoReg, and all-SMAT, respectively.
SMAT's curve rises with each replacement count and remains above the other
curves for $q>0$.
Thus, SMAT can improve an existing FT expert set through partial replacement,
with larger gains than the other methods under this common TA coefficient.

\paragraph{Sources of Merging Gains.}
\label{sec:merge-gain-analysis}
Figure~\ref{fig:expert-analyses}(c) compares expert scores with five-merger
averages from Table~\ref{tab:llm-results}; their difference is the merging drop.
Relative to FT and MergOPT, SMAT has both a higher expert score and a
smaller drop. Relative to OrthoReg, its higher expert score offsets a larger
drop. A smaller merging drop alone therefore does not determine the value
of a training method: SMAT achieves the best mean merged performance by
balancing expert quality and retention after merging.

\input{figures/loss_slices}
\paragraph{Loss Along Merge Directions.}
\label{sec:optimization-landscape}
Figure~\ref{fig:optimization-landscape} probes answer-token negative log-likelihood
(NLL) along two merge directions: rescaling an expert's own update and
adding updates from other tasks.
For target task $i$ and training method $r$, let $\Delta_i^{(r)}$ be its
update and $u_i^{\mathrm{FT}}$ the shared FT partner update. We scale them
by $a$ and $b$, respectively, to define
\begin{equation}
\theta_i^{(r)}(a,b)=\theta_0+a\Delta_i^{(r)}+b u_i^{\mathrm{FT}},\qquad
u_i^{\mathrm{FT}}=\gamma\sum_{j\ne i}\Delta_j^{\mathrm{FT}},\quad\gamma=0.3.
\label{eq:shared-interference-landscape}
\end{equation}
We use the Table~\ref{tab:llm-results} FT, MergOPT, and SMAT
expert sets, with identical FT partner updates and evaluation samples.
The point $(1,0)$ recovers the independent expert.
SMAT shows broader low-loss regions than FT and MergOPT across the six tasks.
These regions are
consistent with Perturb's smoothing of the conditional training loss
(Appendix~\ref{app:perturb-smoothing}), supporting loss smoothing as a mechanism
for improved merging along the tested directions.

\noindent
\begin{minipage}[t]{.49\linewidth}
\vspace{0pt}
\paragraph{Ablation Studies.}
\label{sec:ablation-studies}
We ablate Scale, Mask, and Perturb individually with AdamW in
Table~\ref{tab:smp-ablation}, keeping the training budget and merger settings fixed.
Removing Scale, Mask, or Perturb lowers the Avg. across five mergers by 1.25, 1.07,
and 2.10 points, respectively.
All three operations contribute, with Perturb having the largest effect.
Its importance is consistent with the loss-smoothing interpretation above,
while Scale and Mask provide further gains.
\end{minipage}\hfill
\begin{minipage}[t]{.48\linewidth}
\vspace{0pt}
\input{tables/ablation_summary}

\end{minipage}
\par\medskip

%% file: tables/llm_results.tex
\begin{table}[t]
\centering
\caption{TRACE results with AdamW. Bold marks the best score for each model.}
\label{tab:llm-results}
\vspace{2pt}
\smatresultsetup
\setlength{\extrarowheight}{0.5pt}
\renewcommand{\arraystretch}{1.08}
\setlength{\tabcolsep}{2.0pt}
\begin{tabularx}{\linewidth}{lGcG*{6}{>{\centering\arraybackslash}X}Gcc}
\arrayrulecolor{black}
\toprule
\arrayrulecolor{black!60}
\multirow{2}{*}{Method} & \multirow{2}{*}{Expert $\uparrow$} & \multicolumn{6}{cG}{Merged Model Performance $\uparrow$} & \multicolumn{2}{c}{Training cost $\downarrow$} \\
 & & \smatmergeheading{WA} & \smatmergeheading{TA} & \smatmergeheading{TIES} & \smatmergeheading{DARE} & \smatmergeheading{DELLA} & \smatavgheading & Time\,\smatcostunit{($\times$FT)} & Mem.\,\smatcostunit{(GiB)} \\
\midrule
\multicolumn{10}{@{}c@{}}{\rule[-0.6ex]{0pt}{3.1ex}\textit{Llama-3.2-1B-Instruct}} \\
\midrule
FT & 54.43 & 39.88\kern0.2pt{\scriptsize\color{black!30}$\uparrow$0.0} & 42.64\kern0.2pt{\scriptsize\color{black!30}$\uparrow$0.0} & 42.93\kern0.2pt{\scriptsize\color{black!30}$\uparrow$0.0} & 43.50\kern0.2pt{\scriptsize\color{black!30}$\uparrow$0.0} & 41.17\kern0.2pt{\scriptsize\color{black!30}$\uparrow$0.0} & 42.02\kern0.2pt{\scriptsize\color{black!30}$\uparrow$0.0} & 1.000 & 20.0\kern0.2pt{\scriptsize\color{black!30}$\uparrow$0.00\%} \\
ASAM & 55.20 & 38.05\kern0.2pt{\scriptsize\color{black!60}$\downarrow$1.8} & \textbf{46.29}\kern0.2pt{\scriptsize\color{black!60}$\uparrow$3.6} & 45.48\kern0.2pt{\scriptsize\color{black!60}$\uparrow$2.5} & 45.71\kern0.2pt{\scriptsize\color{black!60}$\uparrow$2.2} & 44.56\kern0.2pt{\scriptsize\color{black!60}$\uparrow$3.4} & 44.02\kern0.2pt{\scriptsize\color{black!60}$\uparrow$2.0} & 2.869 & 22.3\kern0.2pt{\scriptsize\color{black!60}$\uparrow$11.5\%} \\
MergOPT & 56.23 & 39.67\kern0.2pt{\scriptsize\color{black!60}$\downarrow$0.2} & 45.62\kern0.2pt{\scriptsize\color{black!60}$\uparrow$3.0} & 45.12\kern0.2pt{\scriptsize\color{black!60}$\uparrow$2.2} & 45.59\kern0.2pt{\scriptsize\color{black!60}$\uparrow$2.1} & 44.54\kern0.2pt{\scriptsize\color{black!60}$\uparrow$3.4} & 44.11\kern0.2pt{\scriptsize\color{black!60}$\uparrow$2.1} & 1.718 & 22.3\kern0.2pt{\scriptsize\color{black!60}$\uparrow$11.5\%} \\
OrthoReg & 54.44 & \textbf{43.85}\kern0.2pt{\scriptsize\color{black!60}$\uparrow$4.0} & 45.77\kern0.2pt{\scriptsize\color{black!60}$\uparrow$3.1} & 45.93\kern0.2pt{\scriptsize\color{black!60}$\uparrow$3.0} & 45.16\kern0.2pt{\scriptsize\color{black!60}$\uparrow$1.7} & 42.57\kern0.2pt{\scriptsize\color{black!60}$\uparrow$1.4} & 44.66\kern0.2pt{\scriptsize\color{black!60}$\uparrow$2.6} & 5.693 & 28.2\kern0.2pt{\scriptsize\color{black!60}$\uparrow$41.1\%} \\
\rowcolor{citationblue!7}
\textbf{SMAT} & \textbf{56.61} & 42.59\kern0.2pt{\scriptsize\color{black!60}$\uparrow$2.7} & 46.11\kern0.2pt{\scriptsize\color{black!60}$\uparrow$3.5} & \textbf{47.17}\kern0.2pt{\scriptsize\color{black!60}$\uparrow$4.2} & \textbf{46.90}\kern0.2pt{\scriptsize\color{black!60}$\uparrow$3.4} & \textbf{46.09}\kern0.2pt{\scriptsize\color{black!60}$\uparrow$4.9} & \textbf{45.77}\kern0.2pt{\scriptsize\color{black!60}$\uparrow$3.7} & 1.019 & 22.3\kern0.2pt{\scriptsize\color{black!60}$\uparrow$11.5\%} \\
\midrule
\multicolumn{10}{@{}c@{}}{\rule[-0.6ex]{0pt}{3.1ex}\textit{Llama-3.1-8B-Instruct}} \\
\midrule
FT & \textbf{62.88} & 55.43\kern0.2pt{\scriptsize\color{black!30}$\uparrow$0.0} & 56.03\kern0.2pt{\scriptsize\color{black!30}$\uparrow$0.0} & 55.94\kern0.2pt{\scriptsize\color{black!30}$\uparrow$0.0} & 56.46\kern0.2pt{\scriptsize\color{black!30}$\uparrow$0.0} & 56.01\kern0.2pt{\scriptsize\color{black!30}$\uparrow$0.0} & 55.97\kern0.2pt{\scriptsize\color{black!30}$\uparrow$0.0} & 1.000 & 30.7\kern0.2pt{\scriptsize\color{black!30}$\uparrow$0.00\%} \\
ASAM & 62.05 & 57.09\kern0.2pt{\scriptsize\color{black!60}$\uparrow$1.7} & 56.88\kern0.2pt{\scriptsize\color{black!60}$\uparrow$0.9} & 57.61\kern0.2pt{\scriptsize\color{black!60}$\uparrow$1.7} & 57.01\kern0.2pt{\scriptsize\color{black!60}$\uparrow$0.6} & 56.60\kern0.2pt{\scriptsize\color{black!60}$\uparrow$0.6} & 57.04\kern0.2pt{\scriptsize\color{black!60}$\uparrow$1.1} & 2.676 & 43.3\kern0.2pt{\scriptsize\color{black!60}$\uparrow$40.9\%} \\
MergOPT & 61.24 & 56.23\kern0.2pt{\scriptsize\color{black!60}$\uparrow$0.8} & 57.46\kern0.2pt{\scriptsize\color{black!60}$\uparrow$1.4} & 57.97\kern0.2pt{\scriptsize\color{black!60}$\uparrow$2.0} & 57.53\kern0.2pt{\scriptsize\color{black!60}$\uparrow$1.1} & 56.63\kern0.2pt{\scriptsize\color{black!60}$\uparrow$0.6} & 57.16\kern0.2pt{\scriptsize\color{black!60}$\uparrow$1.2} & 1.543 & 38.2\kern0.2pt{\scriptsize\color{black!60}$\uparrow$24.3\%} \\
OrthoReg & 58.72 & 56.75\kern0.2pt{\scriptsize\color{black!60}$\uparrow$1.3} & 58.19\kern0.2pt{\scriptsize\color{black!60}$\uparrow$2.2} & 56.91\kern0.2pt{\scriptsize\color{black!60}$\uparrow$1.0} & 57.90\kern0.2pt{\scriptsize\color{black!60}$\uparrow$1.4} & 57.33\kern0.2pt{\scriptsize\color{black!60}$\uparrow$1.3} & 57.42\kern0.2pt{\scriptsize\color{black!60}$\uparrow$1.4} & 11.58 & 61.5\kern0.2pt{\scriptsize\color{black!60}$\uparrow$100.\%} \\
\rowcolor{citationblue!7}
\textbf{SMAT} & 61.23 & \textbf{57.99}\kern0.2pt{\scriptsize\color{black!60}$\uparrow$2.6} & \textbf{58.72}\kern0.2pt{\scriptsize\color{black!60}$\uparrow$2.7} & \textbf{58.44}\kern0.2pt{\scriptsize\color{black!60}$\uparrow$2.5} & \textbf{58.54}\kern0.2pt{\scriptsize\color{black!60}$\uparrow$2.1} & \textbf{58.75}\kern0.2pt{\scriptsize\color{black!60}$\uparrow$2.7} & \textbf{58.49}\kern0.2pt{\scriptsize\color{black!60}$\uparrow$2.5} & 1.014 & 38.2\kern0.2pt{\scriptsize\color{black!60}$\uparrow$24.3\%} \\
\arrayrulecolor{black}
\bottomrule
\end{tabularx}
\smatrestorerulecolor
\par\vspace{4pt}
\begin{minipage}{\linewidth}
\footnotesize
Expert: mean score before merging. Avg.: mean over five mergers.
Mem.: peak GPU memory. Arrows: change from FT in score points or memory \%.
\end{minipage}
\end{table}

%% file: tables/vit_results.tex
\begin{table}[t]
\centering
\caption{CLIP results with Adam. Bold marks the best score for each model.}
\label{tab:vit-results}
\vspace{2pt}
\smatresultsetup
\setlength{\extrarowheight}{0.5pt}
\renewcommand{\arraystretch}{1.08}
\setlength{\tabcolsep}{2.0pt}
\begin{tabularx}{\linewidth}{lGcG*{6}{>{\centering\arraybackslash}X}Gcc}
\arrayrulecolor{black}
\toprule
\arrayrulecolor{black!60}
\multirow{2}{*}{Method} & \multirow{2}{*}{Expert $\uparrow$} & \multicolumn{6}{cG}{Merged Model Performance $\uparrow$} & \multicolumn{2}{c}{Training cost $\downarrow$} \\
 & & \smatmergeheading{WA} & \smatmergeheading{TA} & \smatmergeheading{TIES} & \smatmergeheading{DARE} & \smatmergeheading{DELLA} & \smatavgheading & Time\,\smatcostunit{($\times$FT)} & Mem.\,\smatcostunit{(GiB)} \\
\midrule
\multicolumn{10}{@{}c@{}}{\rule[-0.6ex]{0pt}{3.1ex}\textit{CLIP ViT-B/32}} \\
\midrule
FT & 89.26 & 67.22\kern0.2pt{\scriptsize\color{black!30}$\uparrow$0.0} & 71.13\kern0.2pt{\scriptsize\color{black!30}$\uparrow$0.0} & 73.05\kern0.2pt{\scriptsize\color{black!30}$\uparrow$0.0} & 71.21\kern0.2pt{\scriptsize\color{black!30}$\uparrow$0.0} & 71.75\kern0.2pt{\scriptsize\color{black!30}$\uparrow$0.0} & 70.87\kern0.2pt{\scriptsize\color{black!30}$\uparrow$0.0} & 1.000 & 5.8\kern0.2pt{\scriptsize\color{black!30}$\uparrow$0.00\%} \\
ASAM & \textbf{90.61} & 67.57\kern0.2pt{\scriptsize\color{black!60}$\uparrow$0.4} & 72.78\kern0.2pt{\scriptsize\color{black!60}$\uparrow$1.6} & 75.05\kern0.2pt{\scriptsize\color{black!60}$\uparrow$2.0} & 72.78\kern0.2pt{\scriptsize\color{black!60}$\uparrow$1.6} & 73.68\kern0.2pt{\scriptsize\color{black!60}$\uparrow$1.9} & 72.37\kern0.2pt{\scriptsize\color{black!60}$\uparrow$1.5} & 1.730 & 6.1\kern0.2pt{\scriptsize\color{black!60}$\uparrow$5.67\%} \\
MergOPT & 90.05 & 68.13\kern0.2pt{\scriptsize\color{black!60}$\uparrow$0.9} & 73.15\kern0.2pt{\scriptsize\color{black!60}$\uparrow$2.0} & 74.82\kern0.2pt{\scriptsize\color{black!60}$\uparrow$1.8} & 72.91\kern0.2pt{\scriptsize\color{black!60}$\uparrow$1.7} & 73.84\kern0.2pt{\scriptsize\color{black!60}$\uparrow$2.1} & 72.57\kern0.2pt{\scriptsize\color{black!60}$\uparrow$1.7} & 1.146 & 6.1\kern0.2pt{\scriptsize\color{black!60}$\uparrow$5.67\%} \\
OrthoReg & 89.36 & 70.74\kern0.2pt{\scriptsize\color{black!60}$\uparrow$3.5} & 72.82\kern0.2pt{\scriptsize\color{black!60}$\uparrow$1.7} & 74.37\kern0.2pt{\scriptsize\color{black!60}$\uparrow$1.3} & 72.85\kern0.2pt{\scriptsize\color{black!60}$\uparrow$1.6} & 70.11\kern0.2pt{\scriptsize\color{black!60}$\downarrow$1.6} & 72.18\kern0.2pt{\scriptsize\color{black!60}$\uparrow$1.3} & 1.125 & 6.5\kern0.2pt{\scriptsize\color{black!60}$\uparrow$13.8\%} \\
\rowcolor{citationblue!7}
\textbf{SMAT} & 89.27 & \textbf{73.83}\kern0.2pt{\scriptsize\color{black!60}$\uparrow$6.6} & \textbf{74.18}\kern0.2pt{\scriptsize\color{black!60}$\uparrow$3.1} & \textbf{75.31}\kern0.2pt{\scriptsize\color{black!60}$\uparrow$2.3} & \textbf{74.18}\kern0.2pt{\scriptsize\color{black!60}$\uparrow$3.0} & \textbf{74.81}\kern0.2pt{\scriptsize\color{black!60}$\uparrow$3.1} & \textbf{74.46}\kern0.2pt{\scriptsize\color{black!60}$\uparrow$3.6} & 1.010 & 6.1\kern0.2pt{\scriptsize\color{black!60}$\uparrow$5.65\%} \\
\midrule
\multicolumn{10}{@{}c@{}}{\rule[-0.6ex]{0pt}{3.1ex}\textit{CLIP ViT-L/14}} \\
\midrule
FT & 93.97 & 80.27\kern0.2pt{\scriptsize\color{black!30}$\uparrow$0.0} & 84.97\kern0.2pt{\scriptsize\color{black!30}$\uparrow$0.0} & 85.51\kern0.2pt{\scriptsize\color{black!30}$\uparrow$0.0} & 84.95\kern0.2pt{\scriptsize\color{black!30}$\uparrow$0.0} & 84.29\kern0.2pt{\scriptsize\color{black!30}$\uparrow$0.0} & 84.00\kern0.2pt{\scriptsize\color{black!30}$\uparrow$0.0} & 1.000 & 65.1\kern0.2pt{\scriptsize\color{black!30}$\uparrow$0.00\%} \\
ASAM & \textbf{94.28} & 81.58\kern0.2pt{\scriptsize\color{black!60}$\uparrow$1.3} & 86.45\kern0.2pt{\scriptsize\color{black!60}$\uparrow$1.5} & 87.18\kern0.2pt{\scriptsize\color{black!60}$\uparrow$1.7} & 86.14\kern0.2pt{\scriptsize\color{black!60}$\uparrow$1.2} & 85.85\kern0.2pt{\scriptsize\color{black!60}$\uparrow$1.6} & 85.44\kern0.2pt{\scriptsize\color{black!60}$\uparrow$1.4} & 2.000 & 66.2\kern0.2pt{\scriptsize\color{black!60}$\uparrow$1.73\%} \\
MergOPT & 94.02 & 80.86\kern0.2pt{\scriptsize\color{black!60}$\uparrow$0.6} & 86.38\kern0.2pt{\scriptsize\color{black!60}$\uparrow$1.4} & 87.29\kern0.2pt{\scriptsize\color{black!60}$\uparrow$1.8} & 86.38\kern0.2pt{\scriptsize\color{black!60}$\uparrow$1.4} & 85.77\kern0.2pt{\scriptsize\color{black!60}$\uparrow$1.5} & 85.34\kern0.2pt{\scriptsize\color{black!60}$\uparrow$1.3} & 1.025 & 66.2\kern0.2pt{\scriptsize\color{black!60}$\uparrow$1.73\%} \\
OrthoReg & 93.87 & 84.61\kern0.2pt{\scriptsize\color{black!60}$\uparrow$4.3} & 86.67\kern0.2pt{\scriptsize\color{black!60}$\uparrow$1.7} & 86.53\kern0.2pt{\scriptsize\color{black!60}$\uparrow$1.0} & 86.29\kern0.2pt{\scriptsize\color{black!60}$\uparrow$1.3} & 83.99\kern0.2pt{\scriptsize\color{black!60}$\downarrow$0.3} & 85.62\kern0.2pt{\scriptsize\color{black!60}$\uparrow$1.6} & 1.035 & 67.9\kern0.2pt{\scriptsize\color{black!60}$\uparrow$4.32\%} \\
\rowcolor{citationblue!7}
\textbf{SMAT} & 93.66 & \textbf{87.05}\kern0.2pt{\scriptsize\color{black!60}$\uparrow$6.8} & \textbf{87.72}\kern0.2pt{\scriptsize\color{black!60}$\uparrow$2.7} & \textbf{88.21}\kern0.2pt{\scriptsize\color{black!60}$\uparrow$2.7} & \textbf{87.69}\kern0.2pt{\scriptsize\color{black!60}$\uparrow$2.7} & \textbf{88.21}\kern0.2pt{\scriptsize\color{black!60}$\uparrow$3.9} & \textbf{87.78}\kern0.2pt{\scriptsize\color{black!60}$\uparrow$3.8} & 1.002 & 66.2\kern0.2pt{\scriptsize\color{black!60}$\uparrow$1.73\%} \\
\arrayrulecolor{black}
\bottomrule
\end{tabularx}
\smatrestorerulecolor
\par\vspace{4pt}
\begin{minipage}{\linewidth}
\footnotesize
Expert: mean score before merging. Avg.: mean over five mergers.
Mem.: peak GPU memory. Arrows: change from FT in score points or memory \%.
\end{minipage}
\end{table}

%% file: tables/muon_results.tex

\centering
\setlength{\abovecaptionskip}{3pt}
\captionof{table}{TRACE with Muon on Llama-1B. Bold marks the best score.}
\label{tab:muon-results}
\vspace{2pt}
\smatresultsetup
\setlength{\tabcolsep}{3pt}
\begin{tabularx}{\linewidth}{lGcG>{\centering\arraybackslash}XGc}
\arrayrulecolor{black}
\toprule
\arrayrulecolor{black!60}
Method & Expert $\uparrow$ & \textbf{Avg.} $\uparrow$ & Time $\downarrow$\,\smatcostunit{($\times$FT)} \\
\midrule
FT & \textbf{55.31} & 38.25\kern0.2pt{\scriptsize\color{black!30}$\uparrow$0.0} & 1.00 \\
ASAM & 54.69 & 40.34\kern0.2pt{\scriptsize\color{black!60}$\uparrow$2.1} & 1.66 \\
MergOPT & 54.94 & 40.22\kern0.2pt{\scriptsize\color{black!60}$\uparrow$2.0} & 1.25 \\
OrthoReg & 50.89 & 43.41\kern0.2pt{\scriptsize\color{black!60}$\uparrow$5.2} & 2.87 \\
\rowcolor{citationblue!7}
\textbf{SMAT} & 55.07 & \textbf{44.15}\kern0.2pt{\scriptsize\color{black!60}$\uparrow$5.9} & 1.00 \\
\arrayrulecolor{black}
\bottomrule
\end{tabularx}
\smatrestorerulecolor
\par\vspace{4pt}
\begin{minipage}{\linewidth}
\footnotesize
\raggedright
Avg.: mean over five mergers. Arrows: change from FT.
Time: training time relative to FT.
\end{minipage}

%% file: figures/expert_analyses.tex
\begin{figure}[!t]
\centering
\includegraphics[width=\linewidth]{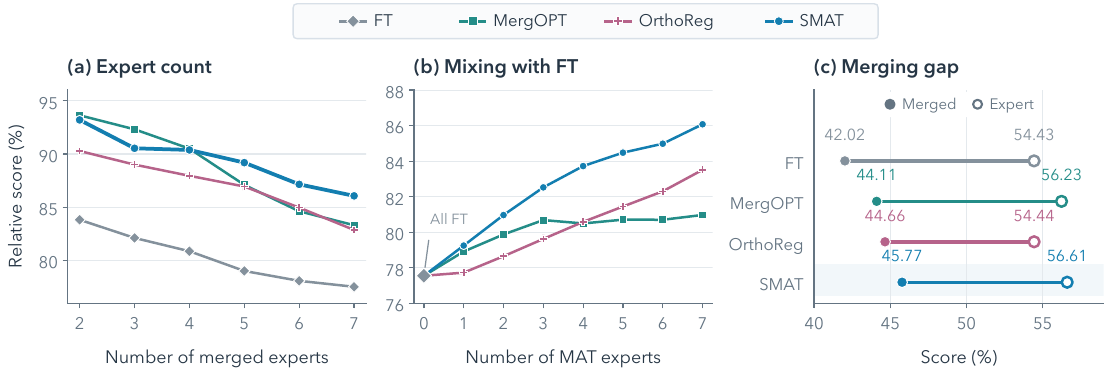}
\caption{Llama-1B expert merging with AdamW: (a) varying expert count,
(b) mixing MAT and FT experts, and (c) expert scores (open) versus mean
merged scores (filled). Relative scores in (a,b) are normalized to FT experts.}
\label{fig:expert-analyses}
\label{fig:expert-count}
\label{fig:expert-mixtures}
\label{fig:merge-gain-decomposition}
\end{figure}

%% file: figures/loss_slices.tex
\begin{figure}[!t]\centering
\includegraphics[width=\linewidth]{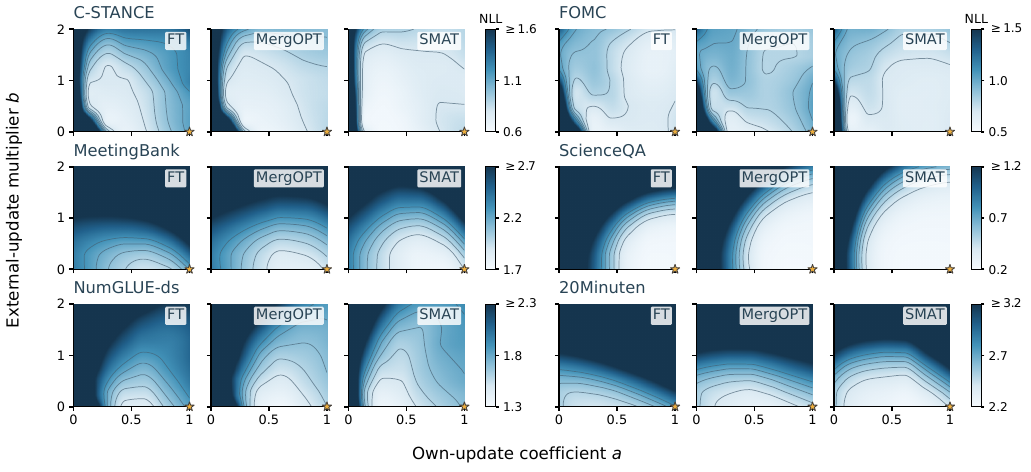}
\caption{Loss slices for six tasks using the experts in Table~\ref{tab:llm-results}.
Columns are FT, MergOPT, and SMAT; stars mark independent experts.
Each task shares one NLL color scale across methods; scales differ across
tasks, and contours are spaced by 0.1.}
\label{fig:optimization-landscape}\end{figure}

%% file: tables/ablation_summary.tex
\centering
\setlength{\abovecaptionskip}{0pt}
\captionof{table}{Component ablation on Llama-1B with AdamW. Avg. averages five mergers; Change is relative to Full SMAT.}
\label{tab:smp-ablation}
\small
\setlength{\tabcolsep}{3pt}
\renewcommand{\arraystretch}{1.0}
\begin{tabularx}{\linewidth}{@{}l>{\raggedleft\arraybackslash}Xr@{}}
\toprule
Variant & \textbf{Avg.} $\uparrow$ & Change \\
\midrule
\rowcolor{citationblue!7}
\textbf{Full SMAT} & $45.77$ & -- \\
w/o Scale & $44.52$ & $-1.25$ \\
w/o Mask & $44.70$ & $-1.07$ \\
w/o Perturb & $43.67$ & $-2.10$ \\
\bottomrule
\end{tabularx}

%% file: sections/06_conclusion.tex
\section{Conclusion}
\label{sec:conclusion}

We introduced SMAT, a simple and efficient merge-aware training method that prepares independently trained experts for model merging by simulating three common operations, i.e., Scale, Mask, and Perturb, with efficient periodic training.
Across language and vision-language models, SMAT consistently improves merged performance across multiple merging methods with less than 2\% training-time overhead.
Our analyses further reveal broader low-loss regions along merge-relevant directions, supporting loss smoothing as a contributing mechanism to SMAT's merge robustness, while Scale and Mask provide complementary gains beyond Perturb.
Overall, our results highlight merge-induced parameter variation as a useful training signal for building experts that remain effective after merging.

%% file: sections/07_statements.tex
\section*{AI use statement}

In this work, we used generative AI tools to assist with language editing
of the manuscript, including polishing wording, improving readability, and
correcting grammar. We also used generative AI tools to assist in drafting
code for functional testing. All AI-assisted text was reviewed and revised
by the authors, and all AI-assisted code was manually inspected and tested
for correctness before use. The authors take full responsibility for the
final content of the manuscript and the correctness of the code used in
this work.

\section*{Reproducibility statement}

Our code is available at
\url{https://github.com/egangu/smat}.
Section~\ref{sec:method} describes SMAT and its efficient implementation.
Section~\ref{sec:experimental-setup} and
Appendix~\ref{app:current-main-settings} provide the experimental settings,
while Appendix~\ref{app:implementation-measurement} describes the
implementation and measurement procedures.
Proofs and their assumptions are provided in
Appendices~\ref{app:perturb-smoothing} and~\ref{app:training-algorithm}.

%% file: sections/appendix/01_experimental_settings.tex
\section{Additional Experimental Settings}
\label{app:current-main-settings}
This appendix specifies the settings for the language and vision experiments
in Tables~\ref{tab:llm-results} and~\ref{tab:vit-results}.

\paragraph{Training.}
For both language backbones, we follow MergOPT~\citep{yang2026mergopt} for
the optimizer, learning rate, batch size, weight decay, and task-specific
epoch counts. We use AdamW with batch size 8, learning rate $2\times10^{-5}$,
a constant schedule, weight decay $10^{-3}$, gradient clipping at 1, and a
maximum length of 512.

The epoch counts for C-STANCE, FOMC, MeetingBank, ScienceQA, NumGLUE-cm,
NumGLUE-ds, and 20Minuten are 5, 3, 7, 3, 5, 5, and 7, respectively. We append
EOS before truncation and supervise it when it is retained.

For both ViT backbones, we use the public CLIP fine-tuning configuration from
FusionBench~\citep{tang2025fusionbench}. Vision models use Adam with learning
rate $10^{-5}$, cosine decay without warmup, zero weight decay, batch size
128, and 4,000 updates per task. The text encoder, projection, and
classification heads remain frozen.

Table~\ref{tab:baseline-training-settings} lists the ASAM~\citep{kwon2021asam} radius $\rho$, the
MergOPT~\citep{yang2026mergopt} Laplace scale $b_{\mathrm{MergOPT}}$, and the
OrthoReg~\citep{liu2026orthoreg} regularization weight $\lambda_{\mathrm{reg}}$.
\begin{table}[htbp]
\centering
\small
\caption{Baseline training hyperparameters for the four backbones.}
\label{tab:baseline-training-settings}
\begin{tabular}{lrrr}
\toprule
Backbone & ASAM radius $\rho$ & MergOPT Laplace scale $b_{\mathrm{MergOPT}}$ & OrthoReg $\lambda_{\mathrm{reg}}$ \\
\midrule
Llama-1B & 2 & $5\times10^{-4}$ & 0.03 \\
Llama-8B & 0.25 & $2.5\times10^{-4}$ & 0.01 \\
ViT-B/32 & 0.25 & $2\times10^{-4}$ & 0.1 \\
ViT-L/14 & 0.125 & $5\times10^{-5}$ & 0.05 \\
\bottomrule
\end{tabular}
\end{table}

\paragraph{SMAT settings.}
For task $i$, SMAT samples one scalar Scale coefficient
$\alpha_i\sim\mathcal{U}[\alpha_{\min},1]$ for all trainable coordinates. It
uses a block-linear Mask with probability $p=0.5$, excludes embeddings from
masking, and simulates merging every four steps.
Both language backbones use $(\alpha_{\min},\sigma)=(0.2,0.002)$, and both
vision backbones use $(0.1,0.001)$, where $\alpha_{\min}$ is the Scale lower
bound and $\sigma$ is the Perturb RMS.

\paragraph{Data splits.}
To construct training, dev, and final evaluation splits, we group examples
by normalized prompt or RGB-image identity,
remove groups present in the final evaluation split from training and dev,
and remove dev groups from training. The three splits therefore have no
overlapping examples under this exact-identity definition.

Language dev sizes are 200, 190, 200, 200, 50, 102, and 200 in the task order
above. Vision uses 200 examples per task, including 200 of the 397 SUN397
classes. We evaluate language tasks on their released eval splits and vision
tasks on their test splits; NumGLUE-cm has 41 evaluation examples.

\paragraph{Sparse-merger calibration.}
We calibrate sparsity on dev data using FT on Llama-1B and ViT-B/32. For
TIES~\citep{yadav2023ties}, retention is selected from
$\{0.1,0.2,0.3,0.5\}$. For DARE~\citep{yu2024dare} and
DELLA~\citep{deep2024della}, drop probability is selected from
$\{0.1,0.3,0.5,0.7,0.9\}$; the DELLA window is fixed at 0.14.

Reference merging coefficients are 0.3 for language TIES/DARE, 0.2 for vision
TIES/DARE, and 1.0 for DELLA; ties favor less pruning. Each selected sparse
setting transfers across methods and model sizes within its family, after
which each expert set selects its own merging coefficient.

TIES retention is 0.5 and DARE drop probability is 0.3 for both families.
DELLA drop probability is 0.5 for language and 0.1 for vision.

\Needspace{6\baselineskip}
\paragraph{Decoding.}
For TRACE dev selection, each candidate gets one decoding run; the top two get
two more runs on MeetingBank, ScienceQA, and 20Minuten, reusing deterministic-task
scores. Selection uses the three-run mean, including for sparse calibration.
Table differences are computed before rounding.

%% file: sections/appendix/02_analysis_settings.tex
\section{Additional Analysis Settings}
\label{app:calibration}

\input{sections/appendix/05_expert_count}

\subsection{Muon Settings and Complete Results}
\label{app:muon-refinement}

\input{tables/muon_results_full}

\paragraph{Training and evaluation.}
All methods share a matrix learning rate of $4\times10^{-4}$ and AdamW at
$2\times10^{-5}$ for auxiliary parameters. SMAT uses scalar
Scale with $\alpha_{\min}=0.2$, Mask probability $p=0.5$, Perturb RMS 0.002,
and interval $t=4$; embeddings are excluded from masking.

We select merger configurations on dev data. All methods
share seven-task data splits and budgets, with batch size 8, maximum length
512, and dynamic padding. Table~\ref{tab:muon-results-full} gives complete results;
Appendix~\ref{app:implementation-measurement} details hardware and cost measurements.

\FloatBarrier

%% file: sections/appendix/05_expert_count.tex
\subsection{Expert-count Evaluation}
\label{app:expert-count-details}

Figure~\ref{fig:expert-analyses}(a) merges the experts from
Table~\ref{tab:llm-results} using Task Arithmetic~\citep{ilharco2023taskarithmetic} (TA).
For each expert count $K=2,\ldots,6$, seven cyclic subsets of a fixed task
ordering balance task inclusion; $K=7$ uses all tasks. Methods share these
subsets and three decoding repeats, and each merged model is evaluated only on
its included tasks. We divide each task score by its independent FT expert score, average over
tasks and subsets, and multiply by 100. 

%% file: tables/muon_results_full.tex
\begin{table}[htbp]
\centering
\caption{Muon results on Llama-3.2-1B-Instruct. Avg. is the mean over five merger methods.}
\label{tab:muon-results-full}
\small
\setlength{\tabcolsep}{2pt}
\renewcommand{\arraystretch}{1.08}
\begin{tabular*}{\linewidth}{@{\extracolsep{\fill}}lrrrrrrr@{}}
\toprule
\multirow{2}{*}{Method} & \multirow{2}{*}{Expert}
& \multicolumn{5}{c}{Merged} & \multirow{2}{*}{Avg.} \\
\cmidrule(lr){3-7}
 & & WA & TA & TIES & DARE & DELLA & \\
\midrule
FT & \textbf{55.31} & 37.06 & 40.46 & 38.53 & 40.92 & 34.26 & 38.25 \\
ASAM & 54.69 & 39.50 & 41.35 & 41.65 & 41.89 & 37.32 & 40.34 \\
MergOPT & 54.94 & 38.07 & 42.52 & 41.90 & 42.02 & 36.59 & 40.22 \\
OrthoReg & 50.89 & 42.60 & 44.73 & 43.98 & 45.42 & \textbf{40.35} & 43.41 \\
\rowcolor{citationblue!7}
\textbf{SMAT} & 55.07 & \textbf{42.77} & \textbf{46.27} & \textbf{46.44} & \textbf{45.72} & 39.55 & \textbf{44.15} \\
\bottomrule
\end{tabular*}\par\vspace{6pt}
\end{table}

%% file: sections/appendix/03_theory.tex
\section{Theory of Simulated Merging}
\label{app:perturb-smoothing}

Fix expert $i$ and the independent draws in Eq.~\eqref{eq:smat-state}, with
coordinate masking probabilities $0\leq p_k<1$ from Eq.~\eqref{eq:smat-mask}.
We use Euclidean vector norms $\|\cdot\|_2$ and operator matrix norms
$\|\cdot\|_{\mathrm{op}}$; expectations average the SMAT draws unless specified.

\begin{proposition}[Moments of the simulated state]
\label{prop:simulated-state-moments}
Let $\bar\alpha:=\mathbb{E}[\alpha_i]=(1+\alpha_{\min})/2$,
$v_\alpha:=\operatorname{Var}(\alpha_i)=(1-\alpha_{\min})^2/12$, and
$s_\alpha:=\mathbb{E}[\alpha_i^2]=\bar\alpha^2+v_\alpha$.  The mean
$\bar\theta_i$ and covariance $C_i$ of the simulated state are
\begin{equation}
  \begin{aligned}
  \bar\theta_i:=\mathbb{E}[\widetilde\theta_i]
  &=\theta_0+\bar\alpha\Delta_i,\\
  C_i:=\operatorname{Cov}(\widetilde\theta_i)
  &=v_\alpha\Delta_i\Delta_i^\top
  +s_\alpha\operatorname{Diag}\!\left(
    \Delta_{i,k}^2\frac{p_k}{1-p_k}\right)+\sigma^2 I_d.
  \end{aligned}
  \label{eq:simulated-state-moments}
\end{equation}
Here $\operatorname{Diag}$ forms a diagonal matrix, $I_d$ is the $d\times d$ identity
matrix, $^\top$ denotes transpose, and $\operatorname{Var}$ and
$\operatorname{Cov}$ denote variance and covariance.
\end{proposition}
\noindent\textit{Proof.}
For each coordinate, $\mathbb{E}[m_{i,k}]=1$ and
$\mathbb{E}[m_{i,k}^2]=1/(1-p_k)$. The scalar $\alpha_i$ is shared across
coordinates, producing the first covariance term. Independent masking
produces the diagonal term, and centered independent Perturb produces
$\sigma^2 I_d$.\hfill$\square$

\paragraph{Local loss expansion.}
Let $Z:=\widetilde\theta_i-\bar\theta_i$ be the centered displacement, and
write $H_i(\theta):=\nabla^2\mathcal{L}_i(\theta)$ for the loss Hessian.
Assume $\mathcal{L}_i$ is twice differentiable and $H_i$ is $L_H$-Lipschitz
on the convex hull of the support of $\widetilde\theta_i$, meaning
$\|H_i(\theta)-H_i(\theta')\|_{\mathrm{op}}\leq
L_H\|\theta-\theta'\|_2$.  Then
\begin{equation}
 \mathbb{E}\mathcal{L}_i(\widetilde\theta_i)
 =\mathcal{L}_i(\bar\theta_i)
 +\frac12\operatorname{tr}\!\left[H_i(\bar\theta_i)C_i\right]+r_i,
 \qquad |r_i|\leq\frac{L_H}{6}\mathbb{E}\|Z\|_2^3.
 \label{eq:simulated-joint-expansion}
\end{equation}
Here $\operatorname{tr}$ is the matrix trace and $r_i$ the third-order Taylor
remainder. Since $\mathbb{E}Z=0$, the linear term vanishes. For positive
semidefinite $H_i(\bar\theta_i)\succeq0$, the trace term penalizes curvature.

\paragraph{Perturb smooths the conditional loss.}
Condition on Scale and Mask, and write
$\theta_\star:=\widetilde\theta_i-\epsilon_i$. For $\sigma>0$, let
$q_\sigma$ be the symmetric Perturb density. The conditional smoothed loss is
\begin{equation}
  \overline{\mathcal{L}}_{i,\sigma}(\theta_\star)
  :=\mathbb{E}_{\epsilon\sim q_\sigma}[\mathcal{L}_i(\theta_\star+\epsilon)]
  =(\mathcal{L}_i*q_\sigma)(\theta_\star),
  \label{eq:perturb-convolution}
\end{equation}
where $*$ denotes convolution. Write $\epsilon=\sigma z$, where $z$ is
standardized noise independent of $\sigma$, with $\mathbb{E}_{z}[z]=0$,
$\mathbb{E}_{z}[zz^\top]=I_d$, and $\mathbb{E}_{z}\|z\|_2^3<\infty$. If
$H_i$ is $L_H$-Lipschitz on every sampled segment from $\theta_\star$ to
$\theta_\star+\sigma z$, then, for sufficiently small $\sigma$,
\begin{equation}
 \overline{\mathcal{L}}_{i,\sigma}(\theta_\star)
 =\mathcal{L}_i(\theta_\star)+\frac{\sigma^2}{2}\operatorname{tr}H_i(\theta_\star)+r_\sigma,
 \qquad
 |r_\sigma|\leq\frac{L_H\sigma^3}{6}\mathbb{E}_{z}\|z\|_2^3=o(\sigma^2).
 \label{eq:perturb-smoothing-expansion}
\end{equation}
Here $r_\sigma$ is the Taylor remainder; $o(\sigma^2)/\sigma^2\to0$ as
$\sigma\to0$. Independent Uniform, Gaussian, and Laplace noise with equal
coordinate variance thus share the leading term, but this expansion does not
rank them at finite noise scales.

\Needspace{10\baselineskip}
\begin{proposition}[Loss difference from a target merge distribution]
\label{prop:simulated-merge-moment-bridge}
Let $P_i$ be a target distribution of merged states (Eq.~\eqref{eq:merged-state})
and $Q_i$ the distribution of $\widetilde\theta_i$. Relative to reference
$\theta_\star$, define the first and second displacement moments and third absolute moment
of a random parameter $\Theta\sim D$:
\begin{equation}
  \begin{aligned}
    a_D&:=\mathbb{E}_{\Theta\sim D}[\Theta-\theta_\star],\\
    S_D&:=\mathbb{E}_{\Theta\sim D}[(\Theta-\theta_\star)(\Theta-\theta_\star)^\top],\\
    \tau_D&:=\mathbb{E}_{\Theta\sim D}\|\Theta-\theta_\star\|_2^3,
    \qquad D\in\{P_i,Q_i\}.
  \end{aligned}
  \label{eq:merge-distribution-moments}
\end{equation}
Assume $\tau_{P_i}$ and $\tau_{Q_i}$ are finite and that $H_i$ is
$L_H$-Lipschitz on the convex hull of
$\{\theta_\star\}\cup\operatorname{supp}(P_i)\cup\operatorname{supp}(Q_i)$,
where $\operatorname{supp}$ denotes support. Then
\begin{equation}
 \begin{aligned}
 &\mathbb{E}_{\Theta\sim P_i}\mathcal{L}_i(\Theta)
 -\mathbb{E}_{\Theta\sim Q_i}\mathcal{L}_i(\Theta)\\
 &\quad=\nabla\mathcal{L}_i(\theta_\star)^\top(a_{P_i}-a_{Q_i})
 +\frac12\operatorname{tr}\!\left[H_i(\theta_\star)(S_{P_i}-S_{Q_i})\right]+r_{P,Q},\\
 &|r_{P,Q}|\leq\frac{L_H}{6}(\tau_{P_i}+\tau_{Q_i}).
 \end{aligned}
 \label{eq:simulated-merge-moment-bridge}
\end{equation}
\end{proposition}
\noindent\textit{Proof of Proposition~\ref{prop:simulated-merge-moment-bridge}.}
Apply Taylor's theorem at $\theta_\star$ under $P_i$ and $Q_i$, then subtract.
Here $r_{P,Q}$ is the combined third-order Taylor remainder. The first term
retains directional or nonzero-mean updates from other experts, so the result
does not assume matching moments.\hfill$\square$

This expected-loss comparison does not guarantee higher task scores.

%% file: sections/appendix/04_training_algorithm.tex
\Needspace{6\baselineskip}
\section{Training and Scheduling Details}
\label{app:training-algorithm}

\Needspace{8\baselineskip}
\paragraph{Proof of Proposition~\ref{prop:periodic-approximation}.}
Let $x_j$ and $y_j$ be the separate-loss and mixed-loss gradient-descent iterates,
respectively, with $x_0=y_0=w$. Let $s_j\in\{A,B\}$ specify the loss
at separate-loss step $j$, with exactly $k$ occurrences of $B$ among
$j=0,\ldots,n-1$. The updates are
\[
 x_{j+1}=x_j-\eta\nabla\mathcal L_{s_j}(x_j),\qquad
 y_{j+1}=y_j-\eta\nabla J(y_j).
\]
Because $J$ is a convex combination of the two losses, its gradient is also
$L_g$-Lipschitz and bounded by $G$; the losses themselves need not be convex.
Thus $\|x_j-w\|_2\leq j\eta G$ and $\|y_j-w\|_2\leq j\eta G$.
The component counts give
$\sum_{j=0}^{n-1}\nabla\mathcal L_{s_j}(w)=n\nabla J(w)$, so
\[
\begin{aligned}
 \|x_n-(w-n\eta\nabla J(w))\|_2
 &\leq\eta\sum_{j=0}^{n-1}
   \|\nabla\mathcal L_{s_j}(x_j)-\nabla\mathcal L_{s_j}(w)\|_2\\
 &\leq\eta L_g\sum_{j=0}^{n-1}j\eta G
 =\frac{\eta^2L_gG\,n(n-1)}{2}.
\end{aligned}
\]
The same bound holds for $y_n$. The triangle inequality gives
Eq.~\eqref{eq:periodic-approximation}.\hfill$\square$

This is a local comparison over one cycle: for fixed $n$, the endpoint
difference is second order in $\eta$. It does not require convex losses
or imply that the two full training trajectories coincide. The result is
for ordinary gradient descent; AdamW's adaptive moments and Muon's momentum
and matrix transformation require a separate analysis.

\paragraph{Cycle-averaged gradients in SMAT.}
The $t-1$ ordinary updates and one simulated-loss update per period match
the objective weight $\lambda=1/t$. The following bound applies to the
population gradients along any parameter path with bounded step lengths,
without assuming a particular optimizer.

\Needspace{10\baselineskip}
\begin{proposition}[Drift within a training period]
\label{prop:periodic-drift}
Let $\zeta_i:=(\alpha_i,\mathbf m_i,\epsilon_i)$ denote an SMAT draw, and
write $\widetilde\theta_i(\theta,\zeta_i)$ for the state in
Eq.~\eqref{eq:smat-state} at current expert parameters $\theta$. Define the
clean and SMAT population gradient fields as
\begin{equation}
  g_{\mathrm c}(\theta):=\nabla\mathcal{L}_i(\theta),\qquad
  g_{\mathrm s}(\theta):=\nabla_\theta\mathbb{E}_{\zeta_i}
  \big[\mathcal{L}_i(\widetilde\theta_i(\theta,\zeta_i))\big].
  \label{eq:periodic-gradient-fields}
\end{equation}
Let $\vartheta_0,\ldots,\vartheta_{t-1}$ be the parameter values in one
period, with $\|\vartheta_{j+1}-\vartheta_j\|_2\leq\delta$ for
$j=0,\ldots,t-2$. Let
$c_j=\mathrm c$ for $0\leq j<t-1$ and $c_{t-1}=\mathrm s$. Suppose both
fields are $L_g$-Lipschitz along this path: for
$h\in\{g_{\mathrm c},g_{\mathrm s}\}$,
$\|h(\theta)-h(\theta')\|_2\leq L_g\|\theta-\theta'\|_2$. Then
\begin{equation}
\left\|\frac{1}{t}\sum_{j=0}^{t-1}g_{c_j}(\vartheta_j)
-\left((1-\tfrac1t)g_{\mathrm c}(\vartheta_0)
+\tfrac1t g_{\mathrm s}(\vartheta_0)\right)\right\|_2
\leq \frac{L_g\delta(t-1)}{2}.
\label{eq:periodic-drift}
\end{equation}
\end{proposition}

\noindent\textit{Proof.}
At $\vartheta_0$, the parenthesized expression equals
$\nabla\mathcal{J}_i^{\mathrm{SMAT}}(\vartheta_0)$ for $\lambda=1/t$.
The triangle inequality and Lipschitz condition bound the difference by
$t^{-1}L_g\sum_{j=0}^{t-1}\|\vartheta_j-\vartheta_0\|_2$, which is at most
$t^{-1}L_g\delta\sum_{j=0}^{t-1}j$.\hfill$\square$

At fixed parameters, unbiased minibatch and simulation gradients also give
the mixed-objective gradient in expectation when averaged over one cycle.
At changing parameters, the bound controls the population-gradient drift;
sampled gradients introduce additional variability. It does not imply
identical updates for stochastic or stateful optimizers.

\paragraph{Implementation details.}
The fused implementation regenerates the forward-pass mask for gradients
and restores original storage before the optimizer update, including any
weight decay.

%% file: sections/appendix/09_implementation_measurement.tex
\clearpage
\section{Implementation and Measurement Details}
\label{app:implementation-measurement}
\label{app:training-time}

\subsection{Hardware, Software, and Implementation}

All training runs use NVIDIA H800 GPUs. Llama-8B uses blockwise model
parallelism across two GPUs per expert while keeping the same effective batch
size and update budget; the other backbones use one GPU. FT--SMAT comparisons match GPU counts and software.

All training runs use Python~3.10.12, PyTorch~2.11.0+cu130, and CUDA~13.0.

SMAT uses the fused implementation from Section~\ref{sec:efficient-training}.
Original and simulated weights occupy separate storage; a reusable buffer
holds simulated weights, and frozen base weights remain in pinned CPU memory.
Fused kernels construct simulated weights and rescale gradients. The model uses simulated storage for the forward and backward pass,
then restores original weights for the optimizer update.

\input{tables/training_time}

\subsection{Training-Time Measurement}

Table~\ref{tab:training-time} reports absolute times for the experiments in
Tables~\ref{tab:llm-results} and~\ref{tab:vit-results}. We measure one full
training run per configuration. Each total is the sum of training-loop
wall-clock seconds across a complete expert set---seven language experts or
eight vision experts---rather than parallel elapsed time.

Mean time per step divides each total by all optimizer steps: 20,123 for a
language expert set and 32,000 for a vision expert set. Timing includes batch
loading, initial warmup, and logging. It excludes initialization, checkpoint
saving, hyperparameter search, merging, and evaluation.

\subsection{Memory Measurement}

For Tables~\ref{tab:llm-results} and~\ref{tab:vit-results}, \emph{Mem.} is
the largest per-GPU PyTorch allocated peak across experts, in GiB. Its
adjacent percentage is the pre-rounding change from FT within the same model
block. These peaks cover training and final checkpoint serialization.

For Muon, memory is the largest per-device PyTorch allocated peak across the
seven experts.
\FloatBarrier

%% file: tables/training_time.tex
\begin{table}[htbp]
\centering
\small
\caption{Training time for complete expert sets. Mean is seconds per optimizer
step; total is summed wall time in seconds.}
\label{tab:training-time}
\setlength{\tabcolsep}{3pt}
\begin{tabular*}{\linewidth}{@{\extracolsep{\fill}}lrrrrrrrr@{}}
\toprule
& \multicolumn{2}{c}{Llama-1B} & \multicolumn{2}{c}{Llama-8B}
& \multicolumn{2}{c}{ViT-B/32} & \multicolumn{2}{c}{ViT-L/14} \\
\cmidrule(lr){2-3}\cmidrule(lr){4-5}\cmidrule(lr){6-7}\cmidrule(l){8-9}
Method & Mean & Total & Mean & Total & Mean & Total & Mean & Total \\
\midrule
FT & 0.0687 & 1,382.1 & 0.3543 & 7,130.1 & 0.1255 & 4,015.3 & 1.4494 & 46,381.4 \\
ASAM & 0.1971 & 3,965.7 & 0.9481 & 19,079.0 & 0.2171 & 6,946.1 & 2.8991 & 92,771.3 \\
MergOPT & 0.1180 & 2,375.2 & 0.5469 & 11,005.0 & 0.1438 & 4,601.5 & 1.4862 & 47,559.9 \\
OrthoReg & 0.3910 & 7,868.5 & 4.1047 & 82,599.3 & 0.1411 & 4,515.7 & 1.5006 & 48,019.5 \\
SMAT & 0.0700 & 1,409.0 & 0.3594 & 7,231.2 & 0.1268 & 4,057.5 & 1.4516 & 46,451.0 \\
\bottomrule
\end{tabular*}
\end{table}

%% file: sections/appendix/10_future_work.tex
\section{Future Work}
\label{app:future-work}

SMAT currently assumes shared expert architectures and initialization.
Model-fusion surveys~\citep{cai2026parameters} and perspectives on accessible,
sustainable AI~\citep{zhou2025democratizing,zhou2026sustainable} motivate
connecting merge-aware training with the broader fusion pipeline.
The following combinations remain to be evaluated.

\paragraph{Scaling and adaptive training.}
Model merging scaling laws~\citep{wang2026scaling} motivate varying model
size, expert count, and domain diversity. MergePipe~\citep{wang2026mergepipe}
and Access Sets Matter~\citep{wang2026access} further suggest evaluating
SMAT under joint compute, memory, and checkpoint-access budgets.
Geometry Conflict~\citep{wang2026geometry} motivates adapting simulation to
update geometry or task history, while preserving independent training
and testing capability retention as experts arrive sequentially.

\paragraph{Calibration, compression, and heterogeneous fusion.}
Combining SMAT with FeatCal~\citep{gu2026featcal} and
E-PMQ~\citep{wang2026epmq} could test whether training-time compatibility
reduces calibration needs or improves low-bit deployment.
InfiFPO~\citep{gu2025infifpo} and InfiGFusion~\citep{wang2025infigfusion}
motivate extending fusion beyond compatible parameters through preference
optimization and distillation. Preference-Aligned Distillation~\citep{gu2025pad}
and Token Teachability~\citep{wang2026teachability} suggest examining which
teacher signals survive transfer. Evaluation should separate each stage's
contribution and report end-to-end cost and preference retention.

\paragraph{Structured knowledge and reliability.}
StructFact~\citep{huang2025structfact}, HyperG~\citep{huang2025hyperg},
the response-uncertainty evaluation of \citet{dang2025uncertainty}, and
NPCBench~\citep{wang2026npcbench} motivate testing structural reasoning,
robustness to misleading multimodal cues, and guideline adherence.
Such evaluations should report per-task regressions, uncertainty, and
constraint violations alongside mean accuracy before drawing conclusions
about deployment reliability.